\documentclass[letterpaper]{article} 
\usepackage{aaai24}  
\usepackage{times}  
\usepackage{helvet}  
\usepackage{courier}  
\usepackage[hyphens]{url}  
\usepackage{graphicx} 
\usepackage{natbib}  
\usepackage{caption} 
\usepackage{algorithm}
\usepackage{algorithmic}

\usepackage{newfloat}
\usepackage{listings}
\DeclareCaptionStyle{ruled}{labelfont=normalfont,labelsep=colon,strut=off} 
\floatstyle{ruled}
\newfloat{listing}{tb}{lst}{}
\floatname{listing}{Listing}
\usepackage{soul}

\newcommand{\RR}{\mathbb{R}}
\newcommand{\NN}{\mathbb{N}}
\usepackage{amsmath}
\usepackage{graphicx}
\usepackage{blkarray}
\usepackage{amsmath, amsthm}
\usepackage{thmtools}
\usepackage{amssymb}
\usepackage{mathtools}
\theoremstyle{plain}
\newtheorem{theorem}{Theorem}[section]
\newtheorem*{theorem*}{Theorem}

\newtheorem{lemma*}{Lemma}

\theoremstyle{definition}
\newtheorem{definition}[theorem]{Definition}

\newtheorem*{claim*}{Claim}

\theoremstyle{remark}
\newtheorem{remark}[theorem]{Remark}
\usepackage{nicematrix}
\usepackage{booktabs}
\usepackage{float}
\usepackage{thmtools} 
\usepackage{thm-restate}
\DeclareMathOperator{\Ima}{Im}
\DeclareMathOperator{\rank}{rank}

\newcommand{\lms}{ \{\!\!\{ }
\newcommand{\rms}{ \}\!\!\} }
\newcommand{\m}{ \mid }
\newcommand{\lb}{ \!\left( }
\newcommand{\rb}{ \!\right) }

\newcommand{\I}{\mathbf{i}}
\newcommand{\Rthree}{\RR^{3 \times n}}
\newcommand{\sort}{\mathrm{sort}}

\newcommand{\embed}{\text{\textbf{Embed}}}

\newcommand{\SES}{\mathcal{SO}[d,n]}

\newcommand{\SESthree}{\mathcal{SO}[3,n]}
\newcommand{\ES}{\mathcal{O}[d,n]}
\newcommand{\ESthree}{\mathcal{O}[3,n]}

\newcommand{\X}{\mathcal{X}}
\newcommand{\Y}{\mathcal{Y}}
\newcommand{\SO}{\mathcal{SO}}
\renewcommand{\O}{\mathcal{O}}
\newcommand{\G}{\mathcal{G}}
\newcommand{\N}{\mathcal{N}}

\newcommand{\Xdistinct}{\RR^{3\times n}_{distinct}}

\title{Complete Neural Networks for Complete Euclidean Graphs}
\author {
    Snir Hordan\textsuperscript{\rm 1},
    Tal Amir\textsuperscript{\rm 1},
    Steven J. Gortler\textsuperscript{\rm 2},
    Nadav Dym\textsuperscript{\rm 1, \rm 3}
}
\affiliations {
    \textsuperscript{\rm 1} Faculty of Mathematics, Technion – Israel Institute of Technology, Haifa, Israel\\
    \textsuperscript{\rm 2}School of Engineering and Applied Sciences, Harvard University, Cambridge, USA\\
    \textsuperscript{\rm 3} Faculty of Computer Science, Technion – Israel Institute of Technology, Haifa, Israel\\
    \{snirhordan, talamir\}@campus.technion.ac.il, sjg@cs.harvard.edu, nadavdym@technion.ac.il
}

\begin{document}

\maketitle

\begin{abstract}
Neural networks for point clouds, which respect their natural invariance to permutation and rigid motion,  have enjoyed recent success in modeling geometric phenomena, from molecular dynamics to recommender systems.  Yet, to date, no model with polynomial complexity is known to be \emph{complete}, that is, able to distinguish between any pair of non-isomorphic point clouds. We fill this theoretical gap by showing that point clouds can be completely determined, up to permutation and rigid motion, by applying the 3-WL graph isomorphism test to the point cloud's centralized Gram matrix.  Moreover, we formulate an Euclidean variant of the 2-WL test and show that it is also sufficient to achieve completeness. We then show how our complete Euclidean WL tests can be simulated by an Euclidean graph neural network of moderate size and demonstrate their separation capability on highly symmetrical point clouds.

\end{abstract}

\section{Introduction}
A point cloud is a collection of $n$ points in $\RR^d$, where typically in applications $d=3$. Machine learning on point clouds is an important task with applications in chemistry \cite{Gilmer2017, wang2022comenet}, physical systems \cite{finzi2021practical}, and image processing \cite{ma2023image}. Many successful architectures for point clouds are invariant by construction to the natural symmetries of point clouds: permutations and rigid motions.

The rapidly increasing literature on point-cloud networks with permutation and rigid-motion symmetries has motivated research aimed at theoretically understanding the expressive power of the various architectures. This analysis typically focuses on two closely related concepts: \emph{Separation} and \emph{Universality}. We say an invariant architecture is \emph{separating}, or \emph{complete}, if it can assign distinct values to any pair of point clouds that are not related by symmetry. An invariant architecture is \emph{universal} if it can approximate all continuous invariant functions on compact sets. Generally speaking, these two concepts are essentially equivalent, as discussed in  \cite{villar,joshi2022expressive,chen2019equivalence}, and in our context, in Appendix A.

 \citet{Dym2020OnTU} proved that the well-known
  Tensor Field Network \cite{thomas2018tensor} invariant architecture is universal, but the construction in their proof requires arbitrarily high-order representations of the rotation group. Similarly, universality can be obtained using high-order representations of the permutation group \cite{lim2022sign}. However, before this work, it was not known whether the same theoretical guarantees can be achieved by
 realistic point-cloud architectures that use low-dimensional representations, and whose complexity has a mild polynomial dependency on the data dimension. In the words of \cite{pozdnyakov2022incompleteness}: \emph{"...provably universal equivariant frameworks are such in the limit in which they generate high-order correlations… It is an interesting, and open, question whether a given order suffices to guarantee complete resolving power."} (p. 6). We note that it is known that separation of point clouds in polynomial time in $n$ is possible, assuming that $d$ is fixed (e.g., $d=3$) \cite{arvind2016parameterized,dym2019linearly,kurlin2024polynomial}. What still remains to be established is whether separation is achievable for common invariant machine learning models, and more generally,  whether separation can be achieved by computing a continuous invariant feature that is piecewise differentiable. 
 
 In this paper, we give what seems to be the first positive answer to this question. We focus on analyzing a popular method for the construction of invariant point-cloud networks via \emph{Graph Neural Networks (GNNs)}. This is done in two steps:  first, point clouds are represented as a \emph{Euclidean graph}-  which we define to be a complete weighted graph whose edge features are simple, rotation-invariant features: the inner products between pairs of (centralized) points. We then apply permutation-invariant \emph{Graph Neural Networks (GNNs)} to the Euclidean graphs to obtain a rotation- and permutation-invariant global point-cloud feature. This leads to a rich family of invariant point-cloud architectures, which is determined by the type of GNN chosen.

The most straightforward implementation of this idea would be to apply the popular message passing GNNs to the Euclidean graphs. One could also consider applying more expressive GNNs. For combinatorial graphs, it is known that message-passing GNNs are only as expressive as the 1-WL graph isomorphism test. There exists a hierarchy of $k$-WL graph isomorphism tests, where larger values of $k$ correspond to more expressive, and more expensive, graph isomorphism tests.  There are also corresponding GNNs that simulate the $k$-WL tests and have an equivalent separation power  \cite{morris2018weisfeiler,prov}. One could then consider applying these more expressive architectures to Euclidean graphs, as suggested by \citet{lim2022sign}. Accordingly, we aim to answer the following questions:

\textbf{Question 1} For which $k$ is the $k$-WL test, when applied to Euclidean graphs, complete?

\textbf{Question 2}
Can this test be implemented in polynomial time by a continuous, piecewise-differentiable architecture?

We begin by addressing Question 1. 
First, we consider a variation of the WL-test adapted for point clouds, which we refer to as \emph{$1$-EWL} (`E' for Euclidean). This test was first proposed by \citet{pozdnyakov2022incompleteness}, where it was shown that it cannot distinguish between all $3$-dimensional point clouds, and consequently, neither can  GNNs like SchNet \cite{schnet}, which can be shown to simulate it. Our first result balances this by showing that two iterations of $1$-EWL are enough to separate \emph{almost any} pair of point clouds.

To achieve complete separation for \emph{all} point clouds, we consider higher-order $k$-EWL tests. We first consider a natural adaptation of $k$-WL for Euclidean graphs, which we name the \emph{Vanilla-$k$-EWL} test. In this test, the standard $k$-WL is applied to the Euclidean graph induced by the point clouds. We show that when $k=3$, this test is complete for $3$-dimensional point clouds. Additionally, we propose a variant of the Vanilla $2$-EWL, which incorporates additional geometric information while having the same complexity. We call this test the \emph{$2$-EWL} test and show that it is complete on 3D point clouds. We also propose a natural variation of $2$-EWL called \emph{$2$-SEWL}, which can distinguish between point clouds that are related by a reflection. This ability is important for chemical applications, as most biological molecules that are related by a reflection are \textit{not} chemically identical \cite{kapon2021evidence} (this molecular property is called \emph{chirality}).

 We next address the second question of how to construct a GNN for Euclidean data with the same separation power as that of the various $k$-EWL tests we describe. For combinatorial graphs, such equivalence results rely on injective functions defined on multisets of discrete features \cite{xu2018powerful}. For Euclidean graphs, one can similarly rely on injective functions for multisets with continuous features, such as those proposed in \cite{dym2023low}. However, a naive application of this approach leads to a very large number of hidden features, which grows exponentially with the number of message-passing iterations (see Figure~\ref{fig:dim}). We show how this problem can be remedied, 
so that the number of features needed depends only linearly on the number of message-passing iterations.

To summarize, our main results in this paper are:
\begin{enumerate}
    \item We show that two iterations of $1$-EWL can separate \emph{almost all} point clouds in any dimension.
    \item We prove the completeness of a single iteration of the vanilla $3$-EWL for point clouds in $\mathbb{R}^3$.
    \item We formulate the $2$-SEWL and $2$-EWL tests, and prove their completeness for point clouds in $\RR^3$.
    \item We explain how to build differentiable architectures for point clouds with the same separation power as Euclidean $k$-WL tests, with reasonable complexity.  
\end{enumerate}

\paragraph{Experiments} We present synthetic experiments that demonstrate that $2$-SEWL can separate challenging point-cloud pairs that cannot be separated by several popular architectures. 

\paragraph{Disambiguation: Euclidean Graphs} In this paper we use a simple definition of a Euclidean graph as the centralized Gram matrix of a point cloud (centralizing the point cloud and then calculating its Gram matrix) and focus on a fundamental theoretical question related to this representation. In the learning literature, terms like `geometric graphs' (not used here) could refer to graphs that have both geometric and non-geometric edge and vertex features, or graphs where pairwise distances are only available for specific point pairs (edges in an
incomplete graph).

\subsection{Related Work}
\label{rlt_wrk}

\paragraph{Euclidean WL}  \citet{pozdnyakov2022incompleteness} showed that $1$-EWL is incomplete for 3-dimensional point clouds. \citet{joshi2022expressive} define separation as a more general definition of geometric graph, which combines geometric and combinatorial features. This work holds various interesting insights for this more general problem but they do not prove completeness as we do here.

\paragraph{Other complete constructions}
As mentioned earlier, \citet{Dym2020OnTU} proved universality with respect to permutations and rigid motions for architectures using high-dimensional representations of the rotation group. Similar results were obtained in \cite{finkelshtein2022simple,gemnet}. In \cite{lim2022sign}, universality was proven for Euclidean GNNs with very high-order permutation representations. In the planar case $d=2$, universality using low-dimensional features was achieved in \cite{bokman2022zz}. For $d\geq 3$ our construction seems to be the first to achieve universality using low dimensional representations.

For general fixed $d$, there do exist continuous algorithms that can separate point clouds up to equivalence in polynomial time, but they do not seem to lend themselves directly to neural architectures. In \cite{kurlin2024polynomial,widdowson2023recognizing} a  polynomial-time algorithm is introduced for computing invariants that completely determine a point cloud that is Lipschitz-continuous, yet these invariants are represented as a `multi-set of multi-sets' and thus do not allow for gradient-descent-based optimization. Our approach is continuous and represents the point cloud as a vector in real space, which allows for back-propagation.
Efficient tests for equivalence of Euclidean graphs were described by \citet{brass2000testing,arvind2016parameterized}, but they compute features that do not depend continuously on the point cloud. 

\paragraph{Weaker notions of universality}  \citet{widdowson2022resolving} suggest a method for distinguishing almost every point clouds up to equivalence, similar to our result here on $1$-EWL.  Similarly, efficient separation/universality can also be obtained for point clouds with distinct principal axes \cite{puny2021frame,kurlin2024polynomial}. Another setting in which universality is easier to obtain is when only rigid symmetries are considered and permutation symmetries are ignored  \cite{wang2022comenet,villar,egnn}. All these results do not provide universality for \emph{all} point clouds, with respect to the joint action of permutations and rigid motions.
\subsection*{Mathematical Notation}
A (finite) \emph{multiset} $\lms y_1,\ldots,y_N \rms $ is an unordered collection of elements where repetitions are allowed.

Let $\G$ be a group acting on a set $\X$. For $X,Y \in \X$,  we say that $X\underset{\G}{=}Y$ if $Y=gX$ for some $g \in \G$. We say that a function $f:\X \to \Y$  is \emph{invariant} if $f(gx)=f(x)$ for all $x\in X, g\in G$. We say that $f$ is \emph{equivariant} if $\Y$ is also endowed with some action of $G$ and  $f(gx)=gf(x) $ for all $x\in \X,g\in \G$. 
A separating invariant mapping is an invariant mapping that is injective, up to group equivalence:
\begin{definition}[Separating Invariant]\label{def:sep}
Let $\G$ be a group acting on a set $\X$. We say $F : \X\to \mathbb{R}^K$ is a \emph{$\G$-separating invariant} with \emph{embedding dimension} $K$ if for all $X,Y \in \X$, $F(X)=F(Y) \Leftrightarrow X\underset{\G}{=}Y$.
\end{definition}
We focus on the case where $\X$ is some Euclidean domain. To enable gradient-based learning, we shall need separating mappings that are continuous everywhere and differentiable almost everywhere.

The symmetry group we consider for point clouds $(x_1,\ldots,x_n)\in \RR^{d\times n} $ is generated by a rotation matrix $R\in \SO(d)$, and a permutation $\sigma \in S_n$. These act on a point cloud by 
$$(R,\sigma)_*(x_1,\ldots,x_n)=(Rx_{\sigma^{-1}(1)},\ldots,Rx_{\sigma^{-1}(n)}) .$$
We denote this group by $\SES$.  In some instances, reflections $R \in \O(d) $ are also permitted, leading to a slightly larger symmetry group, which we denote by $\ES$. Our goal shall be to construct separating invariants for these groups. For the sake of brevity, we do not discuss \emph{translation} invariance and separation, as these can easily be achieved by centering the input point clouds, once $\SES$ (or $\ES$) separating invariants are constructed, see \cite{dym2023low}. 

For simplicity of notation, throughout this paper, we focus on the case  $d=3$. In Appendix A, we explain how our constructions and theorems can be generalized to $d > 3$.
\begin{figure*}[t]
    \centering
    \includegraphics[width=0.75\textwidth]{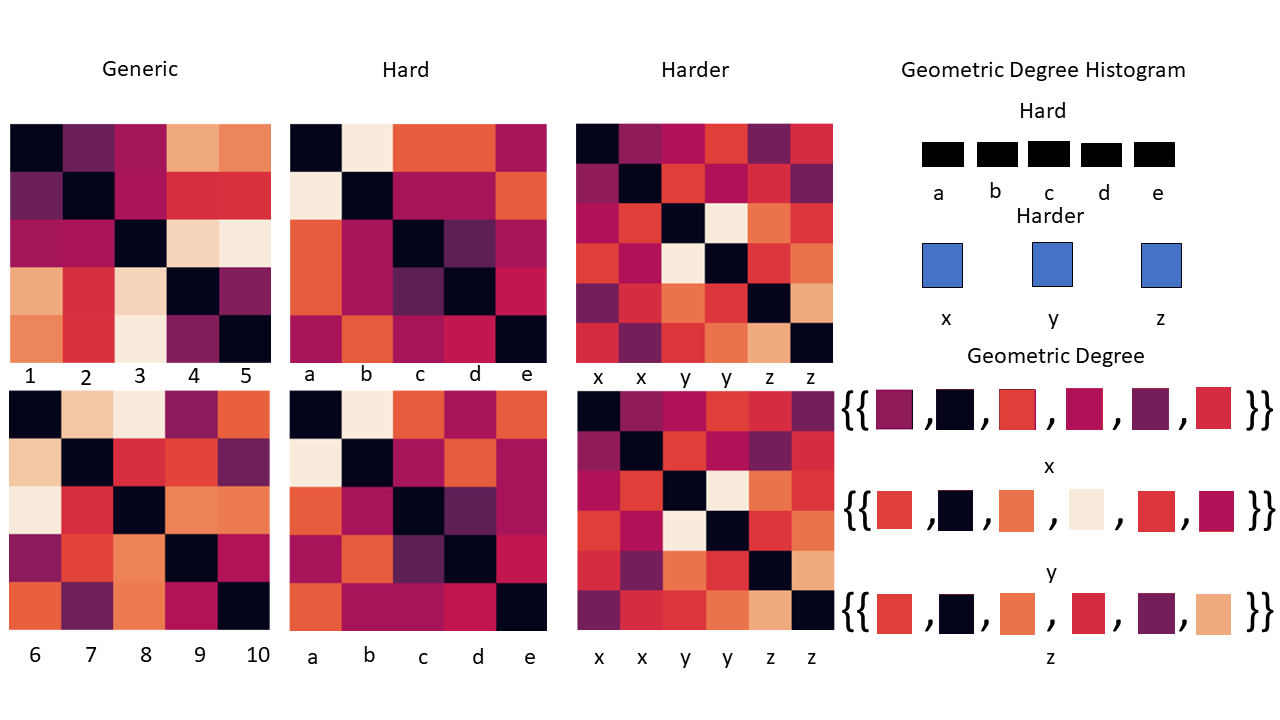}
    \caption{Distance matrices (Left), geometric degree histogram (Right) of pairs of point clouds. The \textit{generic} pair is a randomly sampled pair of point clouds. Notice each of the nodes in each of the clouds has a distinct geometric degree. The \textit{Hard}  pair exhibits a distinct geometric degree for each node, but only within each point cloud, 
 that is the pair shares an identical geometric degree histogram. The \textit{Harder} example is a pair of point clouds with identical geometric degree histogram, and each point cloud is comprised of three pairs of points, with each pair having an identical geometric degree. Examples from \cite{pozdnyakov2022incompleteness} and \cite{pozdnyakov2020incompleteness}.  } 
    \label{fig:degree}
\end{figure*}

\section{Euclidean Graph Isomorphism Tests}\label{sec:tests}
The $k$-WL Graph Isomorphism Test  is a classical paradigm for testing the isomorphism of combinatorial graphs \cite{weisfeiler1968reduction}, which we shall now briefly describe. Let $\G$ be a graph with vertices indexed by $[n] = \{1,2,\ldots,n\}$. We denote each ordered $k$-tuple of vertices by a multi-index $\mathbf{i} = \left( i_1,\ldots,i_k \right) \in [n]^k$. Essentially, for each such $k$-tuple $\mathbf{i}$, the test maintains a \emph{coloring} $\mathbf{C}(\mathbf{i})$ that belongs to a discrete set, and updates it iteratively. First, the coloring of each $k$-tuple is assigned an initial value that encodes the \emph{isomorphism type} of the corresponding $k$-dimensional subgraph: 
\begin{equation}
\mathbf{C_{(0)}}(\mathbf{i}),\: \mathbf{i} \in [n]^k
\end{equation}
Then the color of each $k$-tuple $\mathbf{i}$ is iteratively refined according to the colors of its `neighboring' $k$-tuples. The update rule is defined as
\begin{equation}\label{eq:mpnn}
\mathbf{C_{(t+1)}}(\mathbf{i}) = \embed^{(t+1)}(\mathbf{C_{(t)}}(\mathbf{i}), \lms  N_{j}(\mathbf{C_{(\mathbf{t})}}(\mathbf{i}))  \rms_{j=1}^{n}  )
\end{equation}
where
\begin{equation}
    N_{j}(\mathbf{C_{(t)}}(\mathbf{i})) = \left(\mathbf{C_{(\mathbf{t})}}(\mathbf{i}[j \setminus 1]), \ldots, \mathbf{C_{(\mathbf{t})}}(\mathbf{i}[j \setminus k]) \right)
\end{equation}
where $\mathbf{i}[j \setminus t]$ is the multi-index $\mathbf{i}$ with its $t$-th coordinate replaced by $j$; e.g. for $t=1$, $\mathbf{i}[j\setminus 1] = (j,i_2,\ldots, i_k)$.
$\embed$ is a function that maps its input injectively to some discrete set. This process is repeated $T$ times to obtain a final coloring $ \lms \mathbf{C_{(T)}}(\mathbf{i}) \rms_{\mathbf{i} \in {\left[n\right]^k}}$. A global label is then calculated by
$$\mathbf{C}_{\G}=\embed^{(T+1)} \left( \lms \mathbf{C_{(T)}}(\mathbf{i}) \; | \; \mathbf{i} \in [n]^k \rms \right),$$
where $\embed^{(T+1)}$ is a function that maps label-multisets injectively to some discrete set.

To test whether two graphs $\G$ and $\G'$ are isomorphic, the $k$-WL test computes the corresponding colorings $ \mathbf{C}_{\G}$ and $ \mathbf{C}_{\G'}$ for some chosen $T$. If $\mathbf{C}_{\G} \neq \mathbf{C}_{\G'}$ then $\G$ and $\G'$ are guaranteed not to be isomorphic, whereas if $\mathbf{C}_{\G} = \mathbf{C}_{\G'}$, then $\G$ and $\G'$ may either be isomorphic or not, and the test does not, in general, provide a decisive answer for combinatorial graphs. It is known that this test can distinguish a strictly larger class of combinatorial graphs for every strict increase in the value of k, i.e. it is a strict hierarchy of tests in terms of distinguishing power \cite{CaiFuererImmerman1992, grohe2017descriptive}. We note that in the literature some may refer to the above test as  $k$-\textit{Folklore}-WL, e.g. \cite{morris2018weisfeiler}.

\paragraph{Vanilla-$k$-WL tests} As a first step from a combinatorial to a Euclidean setting, 
we identify each point cloud $X=(x_1,\ldots,x_n)\in \RR^{d \times n} $ with a complete graph on $n$ vertices, wherein each edge $\left(i,j\right)$ is endowed with the weight $w_{ij}(X)=\langle x_i,x_j \rangle$. We name such a graph a \textit{Euclidean graph}. Similarly to $k$-WL for combinatorial graphs, $k$-WL for Euclidean graphs maintains a coloring of the $k$-tuples of vertices. However, the initial color of each $k$-tuple  $\mathbf{i}$ is not a discrete label as in the combinatorial case, but rather a  $k\times k$ matrix of continuous features, which represent all edge weights $w_{ij} $ corresponding to pairs of indices from $\mathbf{i}$. We will call the $k$-WL test defined by this initial coloring the \textbf{vanilla $k$-WL} test. This test is invariant by construction to reflections, rotations, and permutations. We note that our definition of the vanilla $k$-EWL test via inner products follows that of \cite{lim2022sign}. Another popular, and essentially equivalent, formulation, uses distances instead.   

\paragraph{$k$-EWL tests} An inherent limitation of the Vanilla-1-EWL test is that no pairwise Euclidean information is passed, yielding it rather uninformative. Indeed,  \cite{pozdnyakov2022incompleteness} proposed a Euclidean analog of the $1$-WL test, where the update rule $\mathbf{C}_{(\mathbf{t+1})}(i)$
 \eqref{eq:mpnn} is replaced with 
\begin{equation}\label{eq:1wl}
\embed^{(\mathbf{t})} \left( \mathbf{C}_{(\mathbf{t})}(i), \lms \left( \mathbf{C}_{(\mathbf{t})}(j), \|x_i-x_j\| \right),  j\neq i \rms \right)\end{equation}
 We call this test the \textbf{$1$-EWL} test. This formulation is motivated by the fact that many symmetry-preserving networks for point clouds are in fact a realization of it, though they use $\embed$ functions that are continuous and, in general, may assign the same value to different multisets. Consequently, the separation power of these architectures is at most that of $1$-EWL with discrete injective hash functions. Moreover, the separation power will be equivalent if continuous injective multiset functions are used for embedding, as we discuss in the proceeding sections.

The $1$-EWL test strengthens the Vanilla-$1$-EWL test
by allowing the messages passed to a node in each step to contain not only previous colorings but also geometric information in the form of pairwise distances. More generally, we shall use the term \textbf{$k$-EWL} to refer to tests that follow the  Euclidean $k$-WL paradigm, but incorporate geometric invariants into the message-passing procedure. In particular, for point clouds with dimension $3$, we define the $2$-SEWL test (`SE' for \emph{Special Euclidean}) by replacing the update step of \eqref{eq:mpnn}  with


\begin{equation}\label{eq:ewl}
\embed^{(t)} ( \mathbf{C_{(\mathbf{t})}}(i,j), \lms  N^{S}_{k}(\mathbf{C}_{(\mathbf{t})}(i,j))  \rms_{k=1}^{n})     
\end{equation}

where

$$N^{S}_{k}(\mathbf{C}_{(\mathbf{t})}(i,j)) := \left(\mathbf{C}_{(\mathbf{t})}(k,j), \mathbf{C}_{(\mathbf{t})}(i,k)  ,  \langle x_i\times x_j, x_k \rangle \right) $$

 Note that $\langle x_i \times x_j, x_k \rangle$ is equal to the determinant of the $3 \times 3$ matrix whose rows are the three vectors $ x_i, x_j, x_k $, which makes this a natural choice as all polynomial invariants of $\SO(3) $ are generated by these determinants and the inner products we use for the initial coloring \cite{kraft1996classical}.

We note that, Using the fact that $O(3)$ is just two copies of $SO(3)$, it is not difficult to generalize $2$-SEWL to a complete $\ESthree$ test, which we name 2-EWL. for general $d$,  similar complete $(d-1)$-SEWL and $(d-1)$-EWL tests can be formulated for point clouds in $\RR^{d}$ via the Hodge-star operator; see Appendix A for further details.

In the rest of this section, we shall prove that the $2$-SEWL,  and vanilla $3$-EWL tests are complete when applied to $\RR^{3 \times n}$, even when using a single iteration $(T=1)$. We shall also show that two iterations of the $1$-EWL test are complete, except on a set of measure zero.



\subsection{Generic Completeness of 1-EWL}
\label{sec_generic_completeness}
The separation power of 1-EWL is closely linked to the notion of \emph{geometric degree}: For a  point cloud $X=(x_1,\ldots,x_n)$, we define the geometric degree $d(i,X) $ of the $i$-th point, and the induced geometric degree histogram $d_H(X)$, to be the multisets 
$$d(i,X)=\lms \|x_1-x_i\|,\ldots,\|x_n-x_i\| \rms,$$ and $$ \quad d_H(X)=\lms d(1,X),\ldots,d(n,X) \rms.$$

It is not difficult to see that if $d_H(X)\neq d_H(Y)$ then $X$ and $Y$ can be separated by a single $1$-EWL iteration. An example of such a pair is shown on the left of Figure~\ref{fig:degree}. With two $1$-EWL iterations, we show that  can separate $X$ and $Y$ even if $d_H(X) = d_H(Y)$, provided that they both belong to the set of point clouds defined by
$$
\Xdistinct :=\{X \in \RR^{3\times n}\;|\; d(i,X) \neq d(j,X) \;,\;  i\neq j  \}.
$$
Such an example, taken from \cite{pozdnyakov2020incompleteness}, is visualized in the middle column of Figure~\ref{fig:degree}. 
\begin{restatable}{theorem}{onewl}\label{thm:1wl}
 Two iterations of the $1$-EWL test assign two point clouds $\X,Y \in  \Xdistinct$ the same value, if and only if 
 $X\underset{\ESthree}{=}Y $.
\end{restatable}

In the appendix, we show that the complement of $\Xdistinct$ has measure zero. Thus this result complements long-standing results for combinatorial graphs, stating that 1-WL can classify almost all such graphs as the number of nodes tends to infinity \cite{babai1980random}.

The right-most pair of point clouds (`Harder') in Figure~\ref{fig:degree} is taken from \cite{pozdnyakov2022incompleteness}. The degree histograms of these point clouds are identical, and they are not in $\Xdistinct$.  \cite{pozdnyakov2022incompleteness} show that this pair cannot be separated by any number of $1$-EWL iterations.  
\subsection{Is 1-EWL All You Need?}\label{sub:is}
\begin{figure}[t]
\includegraphics[width=5cm]{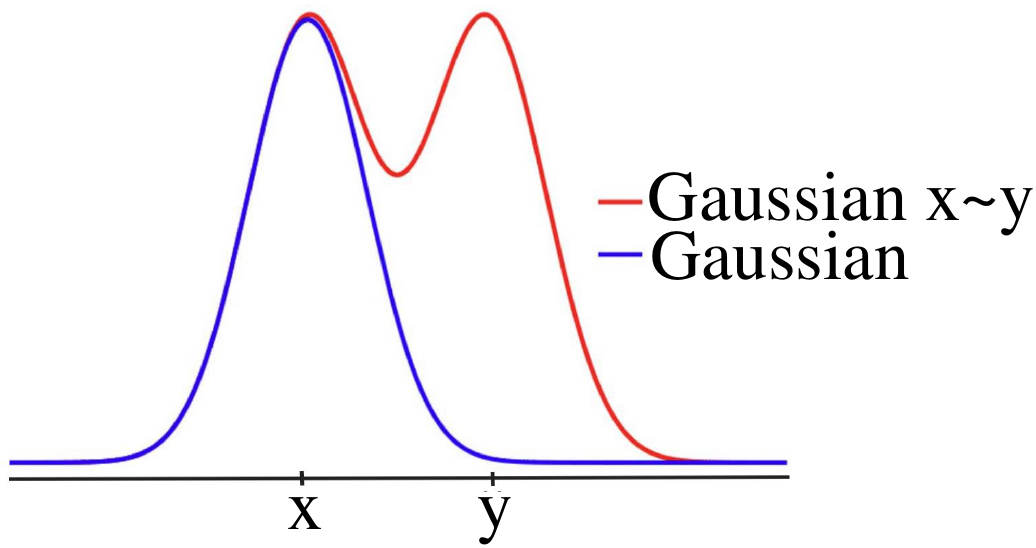}
\centering
\caption{A plot of a Gaussian distribution centered at $x\in \RR$, depicting a target function is shown in blue. In red, a schematic plot of how a Lipschitz continuous function that does not distinguish $x$ from $y$ would model the target function. }
\centering
\end{figure}
Theorem \ref{thm:1wl} shows that the probability of a failure of the $1$-EWL is zero. A natural question to ask is whether more powerful tests are needed. We believe the answer to this question is yes. 

Typical hypothesis classes used for machine learning, such as neural networks, are Lipschitz continuous \cite{Gama_2020}.  In this setting, failure to separate on a measure zero set could have implications for non-trivial positive measure, see Figure 2.


\section{2-SEWL and Vanilla 3-EWL are Complete}
We now prove that the vanilla $3$-EWL test is complete.
\begin{restatable}{theorem}{wl}\label{thm:3wl}
For every $X,Y\in \RR^{3\times n}$, a single iteration of the vanilla  $3$-EWL test assigns $X$ and $Y$ the same value if and only if $X\underset{\ESthree}{=}Y $.  
\end{restatable}
\begin{proof}
First, it is clear that if $X\underset{\ESthree}{=}Y $ then $\mathbf{C_\G}(X)=\mathbf{C_\G}(Y) $ since the vanilla $3$-EWL test is invariant by construction. The challenge is proving the other direction. To this end, let us assume that $\mathbf{C_\G}(X)=\mathbf{C_\G}(Y) $, and assume without loss of generality that $r:=\mathrm{rank}(X)\geq \mathrm{rank}(Y) $. Note that $X$ has rank $r\leq 3$, and so it must contain some three points whose rank is also $r$. By applying a permutation to $X$ we can assume without loss of generality that these three points are the first three points. The initial coloring  $ \mathbf{C_0}(1,2,3)(X)$ of this triplet is their Gram matrix $(\langle x_i,x_j \rangle)_{1\leq i,j \leq 3} $, which has the same rank $r$ as the space spanned by the three points. Next, since  $\mathbf{C_\G}(X)=\mathbf{C_\G}(Y) $ are the same, there exists a triplet of points $i,j,k $ such that $\mathbf{C}_{(1)}(1,2,3)(X)=\mathbf{C}_{(1)}(i,j,k)(Y)$ which implies that the initial colorings are also the same. By applying a permutation to $Y$ we can assume without loss of generality that $i=1,j=2,k=3$. Next, since the Gram matrix of $x_1,x_2,x_3$ and $y_1,y_2,y_3$ are identical, there is an orthogonal transformation that takes $x_i $ to $y_i$ for $i=1,2,3$, and by applying this transformation to all points in $X$ we can assume without loss of generality that $x_i=y_i$ for $i=1,2,3$. It remains to show that the rest of the points of $X$ and $Y$ are equal, up to permutation.  
To see this, first note that $X$ and $Y$ have the same rank since 
$$r=\rank(X)\geq \rank(Y)\geq \rank(y_1,y_2,y_3)$$ $$ =\rank(x_1,x_2,x_3)=r .$$
Thus the space spanned by $x_1=y_1,x_2=y_2,x_3=y_3$ contains all points in $X$ and $Y$. Next, we can deduce from the aggregation rule defining $\mathbf{C_1}(1,2,3)(X)$ in \eqref{eq:mpnn}, that
$$\lms \ (\langle x_j,x_l \rangle)_{l=1}^{3} | \, j \in [n] \rms= \lms ( \langle y_j,y_l \rangle)_{l=1}^{3} | \, j \in [n]  \rms . $$
Since all points in $X$ and $Y$ belong to the span of $x_1=y_1,x_2=y_2,x_3=y_3$, $X$ and $Y$ are the same up to permutation of the last $n-3$ coordinates. This concludes the proof of the theorem.
\end{proof}

We next outline the completeness proof of the more efficient $2$-SEWL.
\begin{restatable}{theorem}{sewl}\label{thm:2sewl}
For every $X,Y\in \RR^{3\times n}$, a single iteration of the $2$-SEWL test assigns $X$ and $Y$ the same value if and only if $X\underset{\SESthree}{=}Y $.  
\end{restatable}
\begin{proof}[Proof idea] The completeness of Vanilla-$3$-EWL was based on the fact that its initial coloring captures the Gram matrix of triplets of vectors that span the space spanned by $X$, and on the availability of projections onto this basis in the aggregation step defined in \eqref{eq:mpnn}. Our proof for $2$-SEWL completeness relies on the fact that a pair of non-degenerate vectors $x_i,x_j $ induces a basis $x_i,x_j,x_i \times x_j$ of $\RR^3$. The Gram matrix of this basis can be recovered from the Gram matrix of the first two points $x_i,x_j $, and the projection onto this basis can be obtained from the extra geometric information we added in \eqref{eq:ewl}. A full proof is given in the appendix.    
\end{proof}
To conclude this section, 
we note that the above theorem can be readily used to show that the $2$-EWL test is also complete with respect to $\ESthree$. For details see Appendix A.

\section{EWL-Equivalent GNNs}\label{sec:gnnslowdim}
In the previous section, we discussed the generic completeness of 1-EWL and the completeness of $2$-SEWL and vanilla $3$-EWL. The $\embed$ functions in these tests are hash functions, which can be redefined independently for each pair of point clouds $X, Y$. In this section, our goal is to explain how to construct GNNs with equivalent separation power to that of these tests, while choosing continuous, piecewise differentiable $\embed$ functions that are injective. While this question is well studied for combinatorial graphs with discrete features \cite{xu2018powerful,morris2018weisfeiler,prov,aamand2022exponentially}, here we focus on addressing it for Euclidean graphs with continuous features.

\subsection{Multiset Injective Functions for Continuous Features}\label{subsec:msinjfun}
\begin{figure}[t] 
    \centering\includegraphics[width=0.43\textwidth]{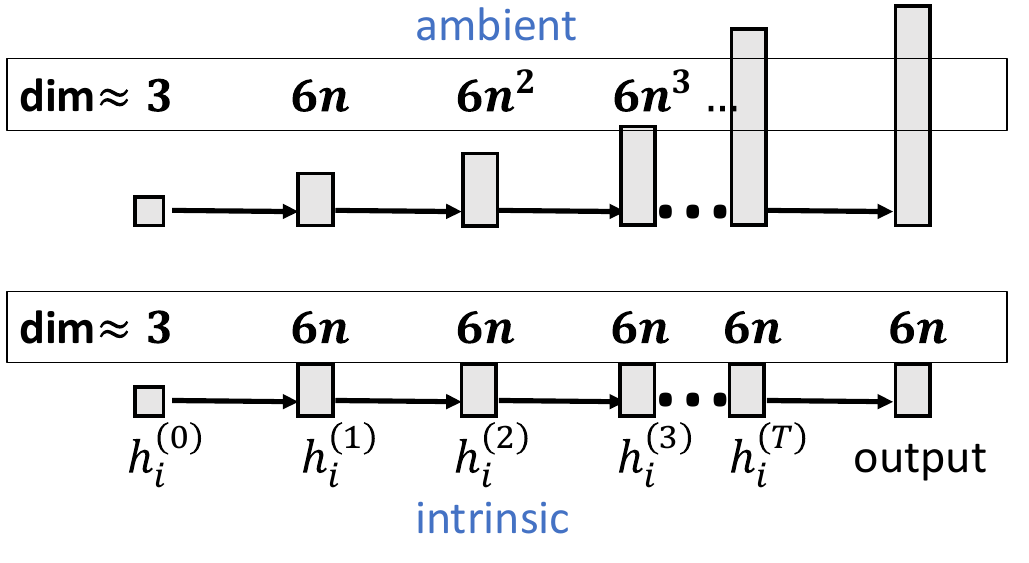}
    \caption{The exponential growth in the dimension that would result from only considering the ambient feature dimension can be avoided by exploiting the constant intrinsic dimension.}
    \label{fig:dim}
\end{figure}
\begin{table*}[ht]
    \centering
    \begingroup
    \renewcommand*{\arraystretch}{1.02}
    \begin{NiceTabular}{cccccc}
        \toprule
       Separation     & complete & $\cong$1-EWL & unknown   & unknown    &  unknown    \\  
        \toprule
        Point Clouds  & 2-SEWLnet   &  1-EWLsim       & MACE & TFN  &  GVPGNN\\
    \cmidrule(lr){1-6}
         \multicolumn{1}{l}{Hard1} & 100 \%   &  100 \%  & 100 \%  & 100 \%   &  100 \%   \\ 
        \multicolumn{1}{l}{Hard2} & 100 \%   &  100 \%  & 100 \%  & 100 \%   &  50  \%   \\ 
         \multicolumn{1}{l}{Hard3} & 100 \%   &  100 \%  & 100 \%  & 100 \%   &  95.0 ± 15.0 \%   \\ 
    \multicolumn{1}{l}{Harder}  &  100 \%   & 50 \%  & 100 \%  & 100 \%  & 53.7 ± 13.1 \%  \\

    \bottomrule
    \end{NiceTabular}
    \endgroup
    \centering
    \caption{ Separation accuracy 
    on challenging 3D point clouds. Hard examples correspond to point clouds which cannot be distinguished by a single  1-EWL iteration but can be distinguished by two iterations, according to  Theorem~\ref{thm:1wl}. The Harder example is a point cloud not distinguishable by 1-EWL  \cite{pozdnyakov2022incompleteness}. GNN implementations and code pipeline based on \cite{joshi2022expressive}.}\label{tab:results}
\end{table*}
Let us first review some known results on injective multiset functions. Recall that a function defined on multisets with $n$ elements coming from some alphabet $\Omega \subseteq \RR^D$ can be identified with a permutation invariant function defined on $\Omega^n$. A multiset function is injective if and only if its corresponding function on $\Omega^n$ is separating with respect to the action of the permutation group (see Definition~\ref{def:sep}). 

In \cite{corso2020principal,wagstaff2022universal} it was shown that for any separating, permutation invariant mappings from $\RR^n$ to $\RR^K$, the embedding dimension $K$ will be at least $n$. Two famous examples of continuous functions that achieve this bound are 
\begin{equation}\label{eq:agg}
\Psi_{sort}( x_1,\ldots,x_n  )=\mathrm{sort}(  x_1,\ldots,x_n  ) \quad  \end{equation}
and 
\begin{equation}
    \quad \Psi_{pow}( x_1,\ldots,x_n )=\left(\sum_{i=1}^n x_i^t \right)_{t=1}^n.
\end{equation}
When the multiset elements are in $\RR^D$, the picture is similar: if there exists a continuous, permutation invariant and separating mapping from  $\RR^{D\times n}$ to $\RR^K$, then necessarily $K\geq n\cdot D$ \cite{joshi2022expressive}. In  \cite{dym2023low} it is shown that continuous separating invariants for $D>1$, with near-optimal dimension, can be derived from the $D=1$ separating invariants  $\Psi=\Psi_{pow}$ or $\Psi=\Psi_{sort}$, by considering random invariants of the form 
\begin{equation}\label{def:embed}
 \embed_\theta( x_1,\ldots,x_n )=\langle b_j,\; \Psi\left(   a_j^Tx_1 \ldots,  a_j^Tx_n  \right) \rangle ,
\end{equation}
where $j=1,\ldots,K$ and each $a_j$ and $b_j$ are $d$ and $n$ dimensional random vectors, respectively, and we denote $\theta=(a_1,\ldots,a_K,b_1,\ldots,b_K)\in \RR^{K(D+n)}$. When $K=2nD+1 $, for almost any choice of $\theta$, the function $\embed_\theta $ will be separating on $\RR^{D \times n}$. Thus the embedding dimension in this construction is optimal up to a multiplicative constant of two. 

An important property of this results of \cite{dym2023low} for our discussion, is that the embedding dimension $K$ can be reduced if the domain of interest is a non-linear subset of $\RR^{D\times n}$ of low dimension. For example, if the domain of interest is a finite union of lines in $\RR^{D\times n}$, then the \emph{instrinsic dimension} of the domain is $1$, and so we will only need an embedding dimension of $K=2\cdot 1+1=3$. Thus, the required embedding dimension depends on the intrinsic dimension of the domain rather than on its \emph{ambient dimension}, which in our case is $n\cdot D $.

To formulate these results precisely we will need to introduce some real algebraic geometry terminology  (see \cite{basu2006algorithms} for more details): A \emph{semi-algebraic subset} of a real finite-dimensional vector space is a finite union of subsets that are defined by polynomial equality and inequality constraints. For example, polygons, hyperplanes, spheres, and finite unions of these sets, are all semi-algebraic sets. A semi-algebraic set is always a finite union of manifolds, and its dimension is the maximal dimension of the manifolds in this union. Using these notions, we can now state the `intrinsic version' of the results in \cite{dym2023low}:
\begin{theorem}[\citet{dym2023low}]\label{thm:lowdim}
Let $\X$ be an $S_n$-invariant semi-algebraic subset of $\RR^{D\times n}$ of dimension $D_\X$. Denote $K=2D_\X+1$. Then for Lebesgue almost every $\theta \in \RR^{K(D+n)} $ the mapping
$\embed_\theta:\X\to \RR^K$
is $S_n$ invariant and separating.
\end{theorem}

\subsection{Multiset Injective Functions for GNNs}\label{sec:injectivegnn}
We now return to discuss GNNs and explain the importance of the distinction between the intrinsic and ambient dimensions in our context. Suppose we are given $n$ initial features $(h_1^{(0)},\ldots,h_n^{(0)}) $ in $\RR^d$, and for simplicity let us assume they are recursively refined via the simple aggregation rule:
\begin{equation}\label{eq:mpnn_simple}
 h_i^{(t+1)} = \embed^{(t)}\lb\lms h_j^{(t)} \rms_{j=1, j\neq i}^{n}\rb. 
\end{equation}
Let us assume that each $\embed^{(t)}$ is injective on the space of all multisets with $n-1$ elements in the ambient space of $h_j^{(t)} $.
Then the injectivity of  $\embed^{(1)}$ implies that $h_i^{(1)}$ is of dimension at least $(n-1)\cdot d$. The requirement that $\embed^{(2)}$ is injective on a mult-set of $n-1$ features in $\RR^{(n-1)\cdot d} $ implies that $h_i^{(2)} $ will be of dimension at least $(n-1)^2\cdot d$. Continuing recursively with this argument we obtain an estimate of $\sim (n-1)^{T}d $ for the dimensions of each $h_i^{(T)}$ after $T$ iterations of \eqref{eq:mpnn_simple}.

Fortunately, the analysis presented above is overly pessimistic because it focused only on the \emph{ambient dimension}. Let us denote the matrix containing all $n$ features at time $t$ by $H^{(t)}$. Then $H^{(t)}=F_t(H^{(0)}) $, where $F_t$ is the  concatenation of all $\embed^{(t')}$ functions from all previous time-steps. Thus $H^{(t)}$ resides in the set $F_t(\RR^{d\times n})$. Here we again rely on results from algebraic geometry: if $F_t $ is a composition of piecewise linear and polynomial mappings, then it is a semi-algebraic mapping, which means that $F_t(H^{(0)})$ will be a semi-algebraic set of dimension $\dim(\RR^{n \times d})= n \cdot d $. This point will be explained in more detail in the proof of Theorem~\ref{thm:6n}. By Theorem~\ref{thm:lowdim}, we can then use $\embed_\theta$ as a multiset injective function on $\X_t$ with a fixed embedding dimension of $2n\cdot d +1$ which does not depend on $T$. This is visualized in Figure~\ref{fig:dim}.


\paragraph{2-SEWLnet} 
Based on the discussion above, we can devise architectures that simulate the various tests discussed in this paper and have reasonable feature dimensions throughout the construction, In particular, we can simulate $T$ iterations of the $2$-SEWL test by replacing all $\embed^{(t)}$ functions\footnotemark \space with  $\embed_\theta^{(t)} $, where in our implementation we choose $\Psi=\Psi_{sort}$ in \eqref{def:embed}. The embedding dimension for all $t$ is taken to be $6n+1$, since the input is in $\RR^{3\times n}$. We denote the obtained parametric function by $F_\phi $. Based on a formalization of the discussion above, we prove in the appendix that $F_\phi$ has the separation power of the complete $2$-SEWL test, and therefore $F_\phi$ is separating.
\begin{restatable}{theorem}{net}\label{thm:6n}
Let $F_\phi$ denote the parametric function simulating the $2$-SEWL test. Then for Lebesgue almost every $\phi$ the function $F_\phi:\RR^{3\times n}\to \RR^{6n+1} $ is separating with respect to the action of $\SESthree$. 
\end{restatable}
\footnotetext{A minor technicality is that the $\embed$ functions are actually defined on vector-multiset pairs. This issue is discussed in the proof of the theorem.}

To conclude this subsection, we note that while sort-based permutation invariants are used as aggregators in GNNs \cite{zhang2020fspool,zhang2018end,blondel2020fast}, the polynomial-based aggregators $\Psi_{pow}$ are not as common. To a certain extent, one can use the approach in \cite{xu2018powerful,prov},  replace the polynomials in $\Psi_{pow}$ by MLPs, and justify this by the universal approximation power of MLPs. A limitation of this approach is that it only guarantees separation at the limit.

\section{Synthetic Experiments}\label{sec:exp}





In this section, we implement 2-SEWLnet and empirically evaluate its separation power, and the separation power of alternative $\SESthree $ invariant point cloud architectures. We trained the architectures on permuted and rotated variations of highly-challenging point-cloud pairs, and measured separation by the test classification accuracy. We considered three pairs of point clouds (Hard1-Hard3) from \citet{pozdnyakov2020incompleteness}. These pairs were designed to be challenging for distance-based invariant methods. However, our analysis reveals that they are in fact separable by two iterations of the 1-EWL test. We then consider a pair of point clouds from \cite{pozdnyakov2022incompleteness} which was proven to be indstinguishable by the 1-EWL tests. The results of this experiment are given in Table~\ref{tab:results}.   Further details on the experimental setup appear in Appendix B. 

%


As expected, we find that $2$-SEWLnet, which has complete separation power, succeeded in perfectly separating all examples. We also found that the simulation of 1-EWL does not separate the \textbf{Harder} example, but \textit{does} separate the \textbf{Hard} example after two iterations, as predicted by Theorem~\ref{thm:1wl}. We also considered three additional invariant point cloud models whose separation power is not as well understood. We find that  MACE \cite{mace} and TFN \cite{thomas2018tensor} achieve perfect separation, (when applying them with at least 3-order correlations and three-order $SO(3)$ representations). The third GVPGNN \cite{jing2021learning} architecture attains mixed results.
We note that we cannot necessarily deduce from our empirical results that MACE and TFN are complete. While it is true that TFN is complete when considering arbitrarily high order representations \cite{Dym2020OnTU}, it is not clear whether order three representation suffices for complete separation. We conjecture that this is not the case. However, finding counterexamples is a challenging problem we leave for future work. 


\subsection*{Future Work} 
In this work,  we presented several invariant tests for point clouds that are provably complete, and have presented and implemented $2$-SEWL-net which simulates the complete $2$-SEWL test. Currently, this is a basic implementation that only serves to corroborate our theoretical results. A practically useful implementation requires addressing several challenges, including dealing with point clouds of different sizes, the non-trivial $\sim n^4 $ complexity of computing even the relatively efficient $2$-SEWL-net, and finding learning tasks where complete separation leads to gains in performance. We are actively researching these directions and hope this paper will inspire others to do the same.    
\section*{Acknowledgements}
This research was supported by the Israel Science Foundation grant no. 272/23.
\bibliography{aaai24}

\appendix
\section*{Appendix}
\section{A :  Proofs}\label{app:proofs}
We begin by stating and proving a result mentioned in the main text: once we construct an invariant separator, we can obtain a universal model by composing the separation with a standard fully connected neural network:  

\begin{theorem}[Separation Implies Universality]
    Let $f: \RR^{d \times n } \to \RR$ be a $G$-invariant continuous function. If $F: \RR^{d\times n} \to \RR^{M}$ is an invariant separator, then for any compact set $K\subset \RR^{d \times n}$, and any $\epsilon>0$, there exists a neural network  $\N^{\epsilon} : \RR^{M} \to \RR$ such that $\underset{x \in K}{\sup}\m f(x) - \N^{\epsilon} \circ F(x) \m < \epsilon$.
\end{theorem}

\begin{proof}
    Let $\epsilon > 0$ and $K\subseteq \mathbb{R}^{d\times n}$ be a compact set.
    Using Proposition 1.3 in \cite{dym2023low}, there exists a continuous $f^{\epsilon}$ such that
    $$ |f(x) - f^{\epsilon}\circ F(x)|<\frac{\epsilon}{2} $$
    The image of a compact set under a continuous function is a compact set, see \cite{munkres2000topology}, then $ S \coloneqq \Ima{(F)}  $ is compact. By the Universal Approximation Theorem \cite{cybenko1989approximation}, we can approximate $f^{\epsilon}$ with a fully-connected Neural Network with arbitrary precision, i.e. there exists a Neural Network Function $\N^{\epsilon}$ such that for all $x\in S$, $$ | f^{\epsilon}(x) - \N^{\epsilon}(x)| < \frac{\epsilon}{2}  $$
    By the Triangle Inequality, for all $x\in K$, 
    $$|f(x) - \N^{\epsilon}\circ F(x) | $$ $$\leq | f(x) - f^{\epsilon}\circ F(x) | + | f^{\epsilon}\circ F(x) - \N^{\epsilon}\circ F(x) | < \frac{\epsilon}{2} + \frac{\epsilon}{2} = \epsilon $$
\end{proof}
\onewl*
\begin{proof}
        Assume we initialize the hidden states with null information. After a single iteration, we have \begin{equation}\label{eqn:oneiter}
        h_{i}^{(1)} = \lms \| x_i - x_j \| \; \m \;  j\neq i \rms = d\lb i, X \rb
    \end{equation} 
    By assumption, all $h_{i}^{(1)}$, i = 1,$\ldots$, n, are distinct. 
    Thus at the next iteration, \begin{equation}\label{eqn:itertwoms}
        h_{i}^{(2)} = \lb h_{i}^{(1)},  \lms (h_{j}^{(1)},  \| x_i - x_j \| ) \; \m \;  j\neq i \rms \rb
    \end{equation}
    
    \begin{equation}\label{eqn:itertwo}
       \lms h^{(2)}_{i} \; \m \; i \in [n] \rms
    \end{equation}
    we know each node's \textit{ordered} distances from the other nodes, as the $i-th$ node is uniquely (intra-point-cloud) determined by $h^{(1)}_i$. Thus we can recover the distance matrix
\[
\begin{blockarray}{cccccc}
 h^{(1)}_1 & h^{(1)}_2 & . & . & h^{(1)}_n \\
\begin{block}{(ccccc)c}
  \|x_1 - x_1\| & \| x_2 - x_1  \| & . & . & \| x_n - x_1  \| & h^{(1)}_1 \\
  \| x_1 - x_2 \| & 0 & . & . & . & h^{(1)}_2 \\
  . & . & 0 & . & . & . \\
  . & . & . & 0 & . & . \\
  \| x_1 - x_n \| & \| x_2 - x_n \| & . & . & \| x_n - x_n \| & h^{(1)}_n \\
\end{block}
\end{blockarray}
 \]
 Thus, we can fully recover the point cloud up to Euclidean motion, see Section E in \cite{egnn}. In conclusion, if $X, Y \in \mathbb{R}^{d \times n}$ are assigned the same value by 1-EWL, then they are identical up to permutation and Euclidean motion, i.e. an $\O[3,n]$ transformation.
\end{proof}
By Theorem~\ref{thm:1wl}, $1$-EWL is complete on $\Xdistinct$. We now show that $1$-EWL is incomplete at most on a (non-trivial) measure-zero set, thus by definition, it is complete almost everywhere on the space of point clouds endowed with permutation, rotation, and reflection symmetries.
\begin{theorem}[1-EWL Separates Almost All Complete Euclidean Graphs]
    Let $\mu$ be the Lebesgue measure on $\RR^{3 \times n}$ where $n\geq 3$. Then, $\mu(\RR^{3\times n} \setminus \Xdistinct) = 0$.     
\end{theorem}
\begin{proof}
    
    We defined $\Xdistinct = \{X \in \RR^{3\times n}\; \m \; d(i,X)\neq d(j,X)\ \  \forall i\neq j  \}$. 
    
    Then 
    $$\Rthree \setminus \Xdistinct = $$ $$ \{ X \in \Rthree \; \m \; \exists i\neq j \; s.t. \; \; d\lb j, X \rb = d\lb i, X \rb \} $$ 
    \begin{equation}\label{eqn:surface}
        = \{ X \in \Rthree \; \m \; \exists i\neq j\in [n] \; s.t. \;cond_{i,j}\;  \text{holds}\}
    \end{equation}
    where $\textit{cond}_{i,j}:= \| \psi_{pow}\lb d(j,X) \rb - \psi_{pow} \lb d(i,X)\rb \|^2 = 0$ and $\psi_{pow}$ is the power-sum polynomials defined as $\psi_{pow}(\Vec{x}) = \lb \sum_{i=1}^{n} x_i, \ldots, \sum_{i=1}^{n} x^n_i \rb$, which is known to be injective on multisets with $n$ elements.

    Equation \ref{eqn:surface} defines an algebraic manifold with a non-trivial polynomial equality constraint, thus is of dimension $\leq 3n - 1$. If an algebraic manifold embedded in $\Rthree$ has dimension $\leq 3n-1$, then it has measure zero \cite{Mityagin2020}.

\end{proof}

\sewl*
 \begin{proof}
    
    Let X, Y $\in \mathbb{R}^{3\times n}$. Recall that a single iteration of the $2$-SEWL assigns to each index pair $i,j$ an initial coloring $\mathbf{C}_{(0)}[i,j]=\mathbf{C}_{(0)}[i,j](X) $ corresponding to the $2\times 2$ Gram matrix of the points $x_i,x_j$. The coloring is then refined via
    $$
\mathbf{C_{(\mathbf{1})}}(i,j) = \embed^{(0)} (  \mathbf{C_{(\mathbf{0})}}(i,j), $$

$$\lms \left( \mathbf{C}_{(\mathbf{0})}(k,j), \mathbf{C}_{(\mathbf{0})}(i,k)  ,   \langle x_i\times x_j, x_k \rangle  \right) \rms_{k=1}^{n}  ). $$
and then a final global coloring is obtained from
$$\mathbf{C}_\G=\embed^{(1)}\left( \lms \mathbf{C}_{(1)}(i,j)| \, (i,j)\in [n]^2   \rms \right) $$
 Let us denote the global feature $\mathbf{C}_\G $ obtained from $X$ by  $\mathbf{C}_\G(X) $, and   the global feature  obtained from $Y$ by  $\mathbf{C}_\G(Y) $.
 
 By construction, if $X\underset{\SESthree}{=}Y$ then $\mathbf{C}_\G(X)=\mathbf{C}_\G(Y)$. We need to prove that if $\mathbf{C}_\G(X)=\mathbf{C}_\G(Y)$ then $X\underset{\ESthree}{=}Y$. 

    To make the proof more readable, we introduce the following notation (we will later describe its significance):

\begin{align}\label{eq:cloud}
X_{[i,j]}&=\left[x_i,x_j,x_i\times x_j \right]\in \RR^{3 \times 3}\\ 
P_{[i,j,k]}&= X_{[i,j]}^T x_k\\
G_{[i,j]}(X)&=X^T_{[i,j]}X_{[i,j]}\\
h_{[i,j]}(X)&=\embed_{\alpha} \lms\label{eqdef_h}
P_{[i,j,k]} \, \mid \, k\in [n], k\neq i,j \rms \\
m_{[i,j]}(X)&=\left( G_{[i,j]}(X) ,h_{[i,j]}(X) \right)
\end{align}
We now show that the multiset $C_\G(X)=\lms \mathbf{C_{(1)}(i,j)} \m i,j \in [n] \rms$ allows recovering the multiset $h_X:=\lms m_{[i,j]} \m i,j \in [n], i,j\in [n] \rms$.

It is enough to show that we can recover $m_{[i,j]}$ from its corresponding $\mathbf{C}_{(1)}[i,j]$ for every $i,j\in [n]$ and then the multiset equivalence follows immediately.
 Note that $G_{[i,j]}$ is the $3\times 3 $ Gram matrix of the vectors $x_i,x_j,x_i \times x_j$. It can be recovered from  $\mathbf{C_{(0)}(i,j)}$, which is the $2\times 2$ Gram matrix of $x_i,x_j$, because  $\langle x_i \times x_j , x_k\rangle=\langle x_i \times x_j , x_i\rangle=0$ and 
$$\| x_i \times x_j\| = \| x_i \| \| x_j\||\sin{\theta}| = \| x_i \| \| x_j\|\sqrt{1- \cos{\theta}^2} $$ $$ =\| x_i \| \| x_j\|\sqrt{1- \lb \frac{\langle x_i, x_j \rangle}{\|x_i\|\|x_j\|} \rb^2},$$
where $\theta$ is the angle between $x_i$ and $x_j$. The quantity on the RHS of the equation can be extracted from $\mathbf{C_{(0)}(i,j)}$.

As for $h_{i,j}$, we can recover it as a multiset since: 
$$\label{eqn:projections}
    P_{[i,j,k]} = \lb \langle x_i, x_k \rangle, \langle x_j, x_k \rangle, \langle x_i \times x_j, x_k \rangle \rb $$ 
    
    $$= \lb \mathbf{C}_{(0)}(i,k)[1,2], \mathbf{C}_{(0)}(k,j)[1,2], \langle x_i \times x_j, x_k \rangle \rb    
$$

We saw that  $h_X=\embed\lms m_{[i,j]}(X) | i,j\in [n] \rms $ can be recovered from $\mathbf{C}_\G(X) $ and thus in particular our assumption that $\mathbf{C}_\G(X)=\mathbf{C}_\G(Y)  $ implies that $h_X=h_Y $. We will now use this to show that $X\underset{\SESthree}{=}Y $. 

We first deal with the degenerate case where all points in $X$ are identical, that is $x_1=x_2=\ldots=x_n$. In this case, all Gram matrices $G_{i,j}(X) $, and all entries in each of the matrices, will be identical, and thus by assumption also all Gram matrices $G_{i,j}(Y)$, and all their entries will be identical. This implies that $Y$ also consists of a single point with the same norm as the one point in $X$, and therefore $X\underset{\SESthree}{=}Y $. 
   
 We can now assume that not all points in $X$ are the same.   Define $r(X) =\mathrm{rank(X)}= \underset{i, j\in[n]}{ \max}\, \rank (G_{[i,j]}(X))$ (note that we assume that $n\geq 3$). By assumption we have $\lms G_{[i,j]}(X) | i,j \in[n] \rms = \lms G_{[i,j]}(Y) | i,j \in[n] \rms $ , thus $r(X) = r(Y)$ and there exist $i\neq j$ and 
 $s, t\in [n]$ such that $G_{[i,j]}(X) = G_{[s,t]}(Y)$, ($\star$)  they both have rank $r$, and $x_i \neq x_j$. Due to the fact that they both have rank r, and 
$x_i \neq x_j$, it follows that $y_s\neq y_t $ and in particular $s\neq t$. 
 
    By \cite{kraft1996classical}, the equality of Gram matrices implies that there exists an orthogonal transformation, $T\in \O(3)$,  such that \begin{equation}\label{eqn:T}
        T(x_i)=y_s, \; T(x_j)=y_t, \; T(x_i \times x_j) = y_s \times y_t.  
    \end{equation}
    If $x_i \times x_s\neq 0$ we see that $T$ preserves orientation and therefore $T\in \SO(3) $ (see Remark~\ref{rem:orientation}). If not, and if  $T$ is a reflection, we can modify $T$ to be a rotation that still satisfies \eqref{eqn:T} by composing it with a reflection that fixes the $\leq 1$ dimensional subspace spanned by $x_i,x_j$. Thus in any case we can assume that $T\in \SO(3)$.  
    By assumption and ($\star$), we have $ \lms P_{[i,j,k]}(X) \, k\neq i,j \rms =  \lms P_{[s,t,k]}(Y), \, k\neq s,t \rms$. \newline
    This implies that there exists some permutation $\sigma\in S_n$ such that $\sigma(i)=s. \sigma(j)=t$ and $P_{[i,j,k]}(X)= P_{[s,t,\sigma(k)]}(Y) $ for all $k\neq i,j$. We deduce that for all $k\neq i,j$
    \begin{align*}
    \langle y_s,Tx_k\rangle=\langle Tx_i,Tx_k\rangle=\langle x_i,x_k\rangle=\langle y_s,y_{\sigma(k)}\rangle\\
    \end{align*}
    and similarly 
    \begin{align*}
    \langle y_t, Tx_k \rangle&=\langle y_t,y_{\sigma(k)}\rangle\\
    \langle y_s\times y_t, Tx_k \rangle&= \langle y_s\times y_t,y_{\sigma(k)}\rangle\\
    \end{align*}
    Now note that each $y_k$ is in the span of $y_s,y_t,y_s\times y_t$ (even when $r<3$), and similarly every $x_k$ is in the span of $x_i,x_j,x_i \times x_j$ and so $Tx_k$ is also in the span of $y_s,y_t,y_s\times y_t$. It follows that $Tx_k-y_{\sigma(k)}=0 $, and thus we showed that $X$ and $Y$ are related by a $\SESthree $ transformation.

\end{proof}
\begin{remark}\label{rem:orientation}
In the proof above we said that if \eqref{eqn:T} holds, and $y_s \times y_t $ is not zero, then $T$ is in $SO(3)$. This follows from the fact that for general orthogonal transformations and vectors $a,b$ 
$$(Ta)\times (Tb)=\det(T)T(a\times b). $$
Setting $a=x_i,b=x_j$ and using \eqref{eqn:T}  we obtain that 
$$y_s\times y_t=\det(T)(y_s \times y_t) $$
and so if $y_s \times y_t$ is not zero, then $\det(T)=1$.
\end{remark}

\net*
\begin{proof}
We recall that $F_\phi$ is defined to simulate a single iteration of the $2$-SEWL test using sort-based injective multiset functions. In more detail, recall that the initial coloring corresponding to an index pair $(i,j)$ and a point cloud $X\in \RR^{3 \times n} $ is given by the Gram matrix of $x_i,x_j$, and denoted by $\mathbf{C}_{(0)}(i,j)=\mathbf{C}_{(0)}(i,j)(X) $. We then define 
$$\mathbf{C_{(\mathbf{1})}}(i,j) =  ( \mathbf{C_{(\mathbf{0})}}(i,j),$$ 

$$ \embed_{\alpha} \left( \left[ \mathbf{C}_{(\mathbf{0})}(k,j), \mathbf{C}_{(\mathbf{0})}(i,k)  ,   \langle x_i\times x_j, x_k \rangle  \right]_{k=1}^{n} \right) ). 
$$and 
$$\mathbf{C}_\G=\embed_{\beta}\left( \mathbf{C}_{(1)}(i,j)| \, (i,j)\in [n]^2  \right) $$
where $\embed_{\theta}(y_1,\ldots,y_n)$ is permutation invariant (=multiset function), continuous in $\theta$ and $y_i$, and defined by
\begin{equation}\label{def:embed2}
\embed_\theta( y_1,\ldots,y_n )=\langle b_j,\; \Psi\left(   a_j^Ty_1 \ldots,  a_j^Ty_n  \right) \rangle 
\end{equation}
with $j=1,\ldots,6n+1$ and $\Psi=\mathrm{sort} $ (or alternatively, $\Psi$ could be the power sum polynomials), and $\theta$ denoting the concatenation of all the mapping parameters $a_i$ and $b_j$. 

We denote by $\phi=(\alpha,\beta)$ the concatenation of the two-parameter vectors of the $\embed$ mappings in the constructions, and $F_\phi(X)$ denoted the output $C_\G=C_\G(X; \phi)$ obtained by this construction.

Since we already showed that the 2-SEWL test is complete, it is sufficient to show that for Lebesgue almost every $(\alpha,\beta)$, the mapping $\embed_\alpha$ is permutation invariant and separating on $\RR^{3\times n} $, and the mapping $\embed_\beta$ is permutation invariant and separating on the image of the mapping $f_\alpha$ which we define as 
$$f_\alpha(X)= \left( \mathbf{C}_{(1)}(i,j)(X;\alpha)| \, (i,j)\in [n]^2  \right) .$$
By Theorem~\ref{thm:lowdim}  we know that $\embed_\alpha$ is separating for Lebesgue almost every $\alpha$. For fixed $\alpha$, we know that $f_\alpha$ is a semi-algebraic mapping, since it is a composition of polynomials and the piecewise linear sort function, which are semi-algebraic mappings, and as compositions of semi-algebraic mappings are semi-algebraic mappings. The dimension of the image of a semi-algebraic mapping is never larger than the dimension of the domain, and so $f_\alpha(\RR^{3\times n}) $ is a semi-algebraic set of dimension $\leq \dim(\RR^{3\times n})=3n $ (see \cite{basu} for the necessary real algebraic geometry statements regarding composition and dimension). To apply Theorem~\ref{thm:2sewl} we need to work with a permutation invariant domain, so we artificially enlarge the domain of  $\embed_\beta$ to be 
$$\bigcup_{\sigma\in S_{n^2}} \sigma(f_\alpha(\RR^{3\times n})) $$
which is a finite union of sets of dimension $\leq 3n$ and hence also has dimension $\leq 3n$. It follows that for almost every $\beta$ the function $\embed_\beta$ is separating  on this permutation invariant set, with embedding dimension of $6n+1$ as we defined in \eqref{def:embed2}. Using Fubini's theorem, this implies that for almost every $(\alpha,\beta)$ the functions $\embed_\alpha$ and $\embed_\beta$ are both separating, and this proves the theorem.
\end{proof}

\paragraph{Complexity}
We conclude by discussing the complexity of computing $F_\phi$. Calculating each $\mathbf{C}_{(1)}(i,j) $ using sort-based embeddings $\embed_\alpha $ requires $\mathcal{O}(n^2\log(n)) $ operations. 

Since there are $\mathcal{O}(n^2)$ such $\mathbf{C}_{(1)}(i,j)$ the total complexity of computing all of them is $\mathcal{O}(n^4\log(n)) $. In the second step we compute $\embed_\beta$ on multisets of size $D\times N $ where $D=\mathcal{O}(n), N=\mathcal{O}(n^2) $, and  with embedding dimension of $\mathcal{O}(n)$. This requires $\mathcal{O}(n^4 + n^3\log(n)) $ operations , so the total complexity is $\mathcal{O}(n^4\log(n)) $


In Appendix~\ref{app:dim} we extend our results to arbitrary $d$. In this case, we get a complexity of  $\mathcal{O}(n^{d+1}\log(n))$ (where for simplicity we consider the limit $n\rightarrow \infty$ with $d$ fixed, to cancel out some mixed terms in $n,d$ which are negligible in this limit. ).
 

\section{B : Experiment Details}\label{app:experiments}
As mentioned, we exemplified the viability of the theory presented by testing separation on challenging point cloud pairs. We wished to address the following scenario: given a pair of point clouds, each labeled distinctly, what would be the accuracy score of $\SESthree$ (or $\ESthree$) invariant architectures in this classification task following training on a labeled dataset of these examples? This setup partly informs us of the \textit{separation} capability of these architectures, i.e. how well do these models distinguish similar ( for instance, 1-EWL equivalent ), yet non-isomorphic, input?

For implementation, we used code by \cite{joshi2022expressive} that implements several contemporary $\SESthree$ invariant architectures and evaluated them as described below. This framework has several additional tests for geometric graphs, but they were irrelevant to our setting because they are redundant for the fully-connected geometric graphs we focus on. We modified the implementation of \cite{joshi2022expressive}  by implementing our novel invariant architectures, 2-SEWLnet, implementing 1-EWLsim, and testing counterexample point cloud pairs from \cite{pozdnyakov2020incompleteness, pozdnyakov2022incompleteness}. We used implementation by \cite{joshi2022expressive}  of MACE \cite{mace}, TFN \cite{thomas2018tensor} and GVPGNN \cite{jing2021learning}. The $\SESthree$ invariant architectures are trained on replicas of each counterexample pair and then testing is performed on the same pair. 

\subsection{Technical Details}
{\centering \begin{table*}[ht]
    \centering
    \begingroup
    \renewcommand*{\arraystretch}{1.02}
    \begin{NiceTabular}{cccccc} 
        \toprule
        Hyperparameters  & 2-SEWLnet   &  1-EWLsim       & MACE & TFN  &  GVPGNN\\
    \cmidrule(lr){1-6}
         \multicolumn{1}{l}{Learning rate} & 0.0001  &  0.0001  & 0.0001  & 0.0001   &  0.0001   \\ 
        \multicolumn{1}{l}{Hidden Dimension} & 2 (Pair-wise)   &  1   & 2  & 64   &  64   \\ 
         \multicolumn{1}{l}{Number of Layers} & 1   &  2  & 3  & 3   &  3  \\ 
    \multicolumn{1}{l}{Batch size}  &  1   & 1  & 1  & 1  & 1  \\
\multicolumn{1}{l}{Correlation}  &  NA   & NA  & 3  & 3  & NA  \\
    \bottomrule
    \end{NiceTabular}
    \endgroup
    \centering
    \caption{ GNN implementations and code pipeline based on \cite{joshi2022expressive}.}\label{tab:config}
\end{table*} \par }

We trained the various invariant models on an  NVIDIA A40 GPU implemented in PyTorch \cite{NEURIPS2019_9015}. The hyperparameters were a learning rate of 0.0001 with Adam optimizer \cite{kingma2017adam}, with the learning rate scheduler ReduceOnPLateau that reduces the learning rate once the loss stopped diminishing. We trained each model on a dataset of 50 copies of each pair for 100 epochs, while injecting permutation and rotation to each point cloud during training. The test and validation datasets are each a pair of plain (no permutation or rotation injected) point clouds. Thus each epoch has ternary accuracy results, of 0\%, 50\%, and 100\%. We then average the accuracy of the last 20\% of epochs to obtain the overall accuracy. This was done to allow the model to converge while allowing for a sufficiently large test measurements to obtain statistical significance.

\subsection{2-SEWLnet}

2-SEWLnet is an implementation of a simulation of a single iteration of 2-SEWL. We implemented the architecture by directly embedding the multiset function $F_\phi$, using the low-dimensional invariant embeddings from the Section Multiset injective functions. We chose the \text{sort} function as the one-dimensional permutation separating invariant, as it constitutes an isometry from $\mathbb{R}^n / S_n$ to $\mathbb{R}^n$, then composing it with linear mappings, yields a Bi-Lipschitz mapping \cite{balan2022permutation, dym2023low}. We used differentiable sorting from PyTorch \cite{NEURIPS2019_9015} to enable backpropagation. We note that the model is able to learn the classification task with backpropagation only using the $\sort$ vector-wise activation, i.e. without a fully-connected neural network composed on it. This is not a trivial result that is immediately implied by the completeness of $2$-SEWL, as we aim to minimize the softmax cross-entropy loss and in practice reach almost zero loss, thus not only yielding two distinct embeddings corresponding to each distinct point cloud  ( as guaranteed by Theorem~\ref{thm:2sewl} ) but learning them to be (approximately) one-hot encoded vectors using our injective continuous $\embed$ functions.

\subsection{1-EWLsim}
 Interestingly,  we found that using the sum aggregation, the model did not yield sufficient separation for the classification task for the examples Hard1-3 in Table~\ref{tab:results}, yet when using $\sort$ and non-linear point-wise activations, rather than sums of neural network functions applied point-wise, led to perfect classification on these counterexamples. The results in the table are reported with the latter implementation. In the code, we denoted this implementation as "egnn", as this simulation is inspired by the invariant version of EGNN \cite{egnn}. 
\section{C : Background Theory}\label{app:msvec}
\subsection{Low Dimensional Separating Invariants}
In the main text, we presented a condensed summary of the results from \cite{dym2023low} as they pertain to this paper's scope. We devote this appendix to expound on the context of these results. In Subsection Multiset injective functions, we discussed the \textit{Power-Sum Symmetric Polynomials}. These polynomials yield a separating invariant with respect to permutations of real-valued vectors in $\mathbb{R}^n$. The invariant learning literature often discussed an extension of this characterization for vector-valued features, the \textit{Multi-Power-Sum Symmetric Polynomials}, defined, for an input $X=\lb x_1 \ldots x_n \rb\in \RR^{d\times n}$ and $\alpha=(\alpha_1,\ldots,\alpha_d)\in \NN_{\geq 0}^d$, as
\begin{align}\label{eqn:mssym}
    P_{\alpha}(X) = \sum_{i=1}^{n} x_{i}^{\alpha} \\
    P(X) = \lb P_{\alpha}(X) \; \rb_{\alpha \in [n]^{d}, \m \alpha \m \leq n}
\end{align}
where $x^{\alpha} :=x_{1}^{\alpha_1}\ldots x_{d}^{\alpha_d}$ and $|\alpha| := \sum_{i=1}^{n} \alpha_i$. These polynomials define a separating invariant for point clouds in $\mathbb{R}^{d}$ with respect to the permutation of the (n) columns \cite{prov}. Yet, its embedding dimension is $\binom{n+d}{d}$. The goal of \cite{dym2023low} was to reduce the embedding dimension to a complexity linear in the $n\cdot d$ dimension of the input.



As a first example of such a result, \cite{dym2023low} show the for Lebesgue almost every $\binom{n+d}{d}$ dimensional vectors $w_1,\ldots, w_{2nd+1}$
\begin{align}\label{eq:sample}
    X \underset{S_n}{=} Y \iff \langle w_i, P(X) \rangle = \langle w_i, P(Y) \rangle, 
\end{align}
where $i=1,\ldots, 2\dim{(\mathcal{M})}+1$.
Thus we obtain a separating invariant of dimension $O(d\cdot n)$ rather than $O(n^{2d})$. Yet, we still had to calculate all of the $O(n^{2d})$ polynomial entries. Therefore, computationally speaking, this approach did not yield much.

To remedy this, \cite{dym2023low} proposed alternative invariants based on  $\RR^n$  permutation invariants $\Psi$. Such a $\Psi$. and any choice of a vector $a\in \RR^d $, induces a permutation invariant on $\RR^{d \times n}$ of the form 
$$\RR^{d\times n}\ni (x_1,\ldots,x_n) \mapsto  \Psi(a^Tx_1,\ldots,a^Tx_n)  .$$
This technique of producing high dimensional invariants from low dimensional ones is known as \emph{polarization}. To obtain invariants that are also \emph{separating}, we need (i) to choose $\Psi$ to be invariant and separating on $\RR^n$. Two examples of such functions, are the functions $\Psi_{sort}$ and $\Psi_{pow}$ which we defined in \eqref{eq:agg}:
\begin{equation*}
\Psi_{sort}( x_1,\ldots,x_n  )=\mathrm{sort}(  x_1,\ldots,x_n  ) \quad  \end{equation*}
and
\begin{equation*}
    \quad \Psi_{pow}( x_1,\ldots,x_n )=\left(\sum_{i=1}^n x_i^t \right)_{t=1}^n.
\end{equation*}
Additionally, we need (ii) to choose not a single polarization function defined by a single $a$, but rather $2nd+1$ random vectors $a_1,\ldots,a_{2nd+1}$. More precisely, \cite{dym2023low} showed that for Lebesgue almost every $a_1,\ldots,a_{2nd+1}\in \RR^d $ and $b_1,\ldots,b_{2nd+1}\in \RR^n $the function
\begin{equation}\label{eq:embed2}
\embed_\theta( x_1,\ldots,x_n )=\langle b_j,\; \Psi\left(   a_j^Tx_1 \ldots,  a_j^Tx_n  \right) \rangle 
\end{equation}
where $, \; j=1,\ldots,2nd+1$, defined in \eqref{def:embed} is permutation invariant and separating. The role of the projection by the $b_j$ is to reduce the embedding dimension to $2nd+1$ rather than the $(2nd+1)n$ dimension we would get if these projections were not applied.

We note that the complexity of a single invariant in \eqref{eq:embed2} would be $O(n\log(n))$ (assuming $n>d$) when using sorting or $O(n^2) $ when using power sum polynomials. Accordingly, the complexity of computing $2nd+1$ invariants would be $O(dn^2\log(n)) $ using sorting or $O(dn^3) $ using power sum polynomials. 

Finally, we note that (as mentioned in the main text), if we're interested in separation only on a semi-algebraic permutation invariant subset  $\X \subseteq \RR^{d\times n}$ with dimension $D_\X$, then the number of separating invariants needed in \eqref{eq:embed2} would be $2D_\X+1 $ rather than $2nd+1 $.

\section{Extensions}\label{app:dim}
In the main text, we described Vanilla $k$-WL tests which are well-defined for all $k$ and $d$, and the $2$-SEWL test which is well-defined for the case $d=3$ (since vector products are used). We now explain how to define a $(d-1) $ SEWL test for general $d$, and then explain how these tests can be easily modified to give a $(d-1) $-EWL test with similar complexity, which is $\ESthree $ invariant and separating rather than $\SESthree $ invariant and separating. 
\subsection{ SEWL for General Dimension $d$ }
The complete 2-SEWL test for $d=3$ point clouds can be generalized to a $(d-1)$-SEWL test for general $d$-dimensional point clouds by a  generalization of the cross-product operator. This generalization is formally known as the Hodge dual operator \cite{jost2017riemannian}. 

For fixed $x_{i_1},\ldots, x_{i_{d-i}} \in \mathbb{R}^d$, we define  the following linear functional 
$$
    f: \mathbb{R}^d \to \mathbb{R}$$
    $$f(x) = \det{(x, x_{i_{1}}, \ldots, x_{i_{d-1}} )}$$
    By Riesz's Representation Theorem \cite{bachman2000functional},  every  linear functional on $\mathbb{R}^d$ is essentially a function in the form of a dot product against some (unique) vector in $  \mathbb{R}^d$. 
 Thus, there exists some vector, denoted by $x^{\star}=x^{\star}(x_{i_1}, \ldots, x_{i_{d-1}}) \in \mathbb{R}^d$, such that 
 \begin{equation}\label{eqn:result}
     \det{(x, x_{i_{1}}, \ldots, x_{i_{d-1}} )}=f(x) = \langle x, x^{\star} \rangle.
 \end{equation}
 We see directly from the definition that $x^{\star}$ is orthogonal to $x_{i_1}, \ldots, x_{i_{d-1}} $ and is non-zero if and only if these $d-1$ vectors are linearly independent. Moreover, this choice of vector is $SO(d)$ equivariant, which means that for any $x_{i_1},\ldots, x_{i_{d-i}} \in \mathbb{R}^d$ and $R \in SO(d)$ we have 
 \begin{align}\label{eqn:orientcrosprod}
    x^{\star}(Rx_{i_1}, \ldots, Rx_{i_{d-1}}) = Rx^{\star}(x_{i_1}, \ldots, x_{i_{d-1}}).
\end{align}



    
This is because for any $x\in \RR^d$ we have 
\begin{align*}
  \det{(x, Rx_{i_{1}}, \ldots, Rx_{i_{d-1}} )}\\
  =\det{(RR^Tx, Rx_{i_{1}}, \ldots, Rx_{i_{d-1}} )}\\
  = \det(R)\cdot \det{(R^Tx, x_{i_{1}}, \ldots, x_{i_{d-1}} )}\\
  = \langle R^Tx, x^{\star} \rangle\\
  = \langle x, Rx^{\star} \rangle
\end{align*}
Finally, we note that the coordinates of $x^*$ can be calculated by inserting the unit vectors $e_1,\ldots,e_d $ into \eqref{eqn:result}. That is
$$
\lb x^{\star}\rb_{j} = \det{\lb e_j, x_{i_1},\ldots, x_{i_{d-1}}  \rb}
$$
where $\lb x^{\star}\rb_{j}$ is the $j$-th entry of $x^{\star}$.
This requires computing $d$ different determinants of $d\times d$ matrices, and so the total complexity of computing $x^*$ is $d^4 $. 
\paragraph{The $(d-1)$-SEWL test} We have shown an extension of the definition of the cross-product that respects orientation, thus now we can naturally define a $(d-1)$-SEWL test which will be $\SES$ invariant and separating in $d$ dimensions (generalizing the $2$-SEWL test for $d=3$).

We define for each $(d-1)$-tuple $\mathbf{i} \in [n]^{d-1}$ an initial coloring $\mathbf{C}_{(0)}(\I)=\mathbf{C}_{(0)}(\I)(X) $ corresponding to the $(d-1)\times (d-1)$ Gram matrix of the points $x_{i_1},\ldots,x_{i_{d-1}}$. We denote $x^*(x_{i_1},\ldots,x_{i_{d-1}}) $ by $x^*(\mathbf{i})$.  The coloring is then refined via
    $$
\mathbf{C_{(\mathbf{1})}}(\I) = \embed^{(0)} (  \mathbf{C_{(\mathbf{0})}}(\I), $$ $$\lms \left( \mathbf{C}_{(\mathbf{0})}(\I [k \setminus 1]), \ldots, \mathbf{C}_{(\mathbf{0})}(\I [k \setminus (d-1)])  ,   \langle x^{\star}(\mathbf{i}), x_k \rangle  \right) \rms_{k}) 
$$

where $\mathbf{i}[j \setminus t]$ is the multi-index $\mathbf{i}$ with its $t$-th coordinate replaced by $j$; e.g. for $t=1$, $\mathbf{i}[j\setminus 1] = (j,i_2,\ldots, i_k)$. Then a final global coloring is obtained from
$$\mathbf{C}_\G=\embed^{(1)}\left( \lms \mathbf{C}_{(1)}(\I)| \, \I \in [n]^{d-1}   \rms \right) $$

The $(d-1)$-SEWL test can be shown to be $\SES$ complete, using the same arguments used in the proof of Theorem~\ref{thm:2sewl}. 


\subsection{ $(d-1)$-EWL for general dimension $d$ }
We now return to the case where reflections are also considered symmetries, and we're looking for complete tests with respect to the group $\ES$. The Vanilla-$d$-WL test will be $\ES$ complete. However, a more efficient test can be obtained by tweaking the $(d-1)$-SEWL test which is not reflection-invariant, to attain a reflection invariant  $\ES$ complete test. 

This tweaking is  obtained as follows. We fix some reflection $R_0$ (a reflection is an orthogonal matrix with a negative determinant).  We define the $(d-1)$-EWL test for a given $X\in \RR^{d\times n} $ by applying the $(d-1)$-SEWL test to both $X$ and $R_0X $ to obtain $\mathbf{C}_\G(X)$ and $\mathbf{C}_\G(R_0X)$, and then computing a final global feature via
\begin{equation}\label{eqn:readoutref}
    C^{\mathit{ref}}_\G(X) = \embed^{(2)} \lms \mathbf{C}_\G(X), \mathbf{C}_\G(R_0X)  \rms
    \end{equation}
In the following theorem, we show how the completeness of the $(d-1)$-SEWL test implies the completeness of the $(d-1)$-EWL test.
\begin{theorem}
For every $X,Y\in \RR^{d\times n}$, a single iteration of the $(d-1)$-EWL test assigns $X$ and $Y$ the same value if and only if $X\underset{\ES}{=}Y $. 
\end{theorem}
\begin{proof}

\textbf{Invariance:}
We prove that for every $R\in O(d)$ and permutation matrix $P\in S_n $ we have that $C^{\mathit{ref}}_\G(RXP)=C^{\mathit{ref}}_\G(X) $. We can divide into two cases: 
If $R\in SO(d) $ then by the  $\SES $ invariance of $\mathbf{C}_\G(X) $ we have that 
\begin{align*}
\mathbf{C}_\G(RXP)&=\mathbf{C}_\G(X)  \\
\mathbf{C}_\G(R_0RXP)&=\mathbf{C}_\G((R_0RR_0^T)R_0XP)=\mathbf{C}_\G(R_0X) 
\end{align*}

On the other hand, if $R$ is a reflection, then 
\begin{align*}
\mathbf{C}_\G(RXP)&=\mathbf{C}_\G((RR_0^T)R_0XP)=\mathbf{C}_\G(R_0X)\\
\mathbf{C}_\G(R_0RXP)&=\mathbf{C}_\G(X) 
\end{align*}
and so in both cases, we obtain that
$$  C^{\mathit{ref}}_\G(RXP) = \embed^{(2)} \lms \mathbf{C}_\G(RXP), \mathbf{C}_\G(R_0RXP)  \rms$$ 

$$= \embed^{(2)} \lms \mathbf{C}_\G(X), \mathbf{C}_\G(R_0X)  \rms= C^{\mathit{ref}}_\G(X) $$




\textbf{Completeness:}
We prove that if $X,Y\in \RR^{d \times n}$ and $C^{\mathit{ref}}_\G(X)=C^{\mathit{ref}}_\G(Y) $  then $X$ and $Y$ are related by a permutation and orthogonal transformation. 

Since $C^{\mathit{ref}}_\G(X)=C^{\mathit{ref}}_\G(Y) $ it follows that either $C_\G(X)=C_\G(Y) $ or $C_\G(X)=C_\G(R_0Y) $ . The completeness of the $(d-1)$-SEWL test (Theorem~\ref{thm:2sewl}) then implies that $X $ is related to either $Y$ or $R_0Y $ by an $\SES$ transformation. In either case, this implies that $X$ and $Y$ are related by an $\ES$ transformation. 
\end{proof}

\subsection{Continuous Implementation and Computational Complexity}
In Section WL-equivalent GNNs with continuous features, we showed how the $2$-SEWL test can be realized by a continuous piecewise differentiable architecture that uses sort-based multi-set injective functions. The complexity of this construction was $O\left(n^{4}\log(n) \right) $. Similarly, the $(d-1)$-SEWL and  $(d-1)$-EWL tests can be computed with complexity of $O\left(n^{d+1}\log(n) \right) $ (in the scenario where $d$ stays constant and $n\to \infty$) . The leading order of the computation complexity stems from computing the $n^{d-1}$ colorings $\lb\mathbf{C_{(1)}(i)}, \; \I \in [n]^{d-1}\rb$ and embedding the multiset derived by aggregating over the `neighbors' of each tuple, each one of those requires $\O(n^2\log(n))$ operations. 

\end{document}